\documentclass{article}

\usepackage{microtype}
\usepackage{graphicx}
\usepackage{subfigure}
\usepackage{booktabs} 
\usepackage[ruled,vlined]{algorithm2e} 
\usepackage{algpseudocode}
\usepackage{xcolor}
\usepackage{amsmath}
\usepackage{listings}
\usepackage{mdframed}
\usepackage{float}

\surroundwithmdframed[
  hidealllines=true,
  innerleftmargin=0pt,
  innertopmargin=0pt,
  innerbottommargin=0pt]{lstlisting}
\lstdefinelanguage{Python}{
keywords={batch_shape, n, x, nblocks, r_blk, blk_sz, blkdiag1, blkdiag2, batch_dim, torch, empty, bmm, transpose, reshape, permute},
keywordstyle=\color{black},
comment=[l]{\#},
commentstyle=\color{commentcolor},
stringstyle=\color{codecolor},
showstringspaces=false,
basicstyle=\ttfamily\footnotesize,
morestring=[b],
}
\definecolor{commentcolor}{RGB}{110,154,155}   
\definecolor{codecolor}{RGB}{0,0,0}            

\usepackage{hyperref}

\usepackage[accepted]{icml2024}

\usepackage{amsmath}
\usepackage{amssymb}
\usepackage{mathtools}
\usepackage{amsthm}
\usepackage{arydshln} 

\usepackage[capitalize,noabbrev]{cleveref}

\theoremstyle{plain}
\newtheorem{theorem}{Theorem}[section]

\newtheorem{lemma}[theorem]{Lemma}
\newtheorem{corollary}[theorem]{Corollary}
\theoremstyle{definition}

\theoremstyle{remark}

\usepackage[textsize=tiny]{todonotes}

\icmltitlerunning{MoRe Fine-Tuning with 10x Fewer Parameters}

\begin{document}
\twocolumn[
\icmltitle{MoRe Fine-Tuning with 10x Fewer Parameters}




\begin{icmlauthorlist}
\icmlauthor{Wenxuan Tan}{sch}
\icmlauthor{Nicholas Roberts}{sch}
\icmlauthor{Tzu-Heng Huang}{sch}
\icmlauthor{Jitian Zhao}{sch}
\icmlauthor{John Cooper}{sch}
\icmlauthor{Samuel Guo}{sch}
\icmlauthor{Chengyu Duan}{sch}
\icmlauthor{Frederic Sala}{sch}
\end{icmlauthorlist}

\icmlaffiliation{sch}{Department of Computer Sciences, University of Wisconsin—Madison, Madison, WI, USA}

\icmlcorrespondingauthor{Wenxuan Tan}{wtan45@wisc.edu}

\icmlkeywords{Machine Learning, ICML}

\vskip 0.3in
]



\printAffiliationsAndNotice{}  
\begin{abstract}
  Parameter-efficient fine-tuning (PEFT) techniques have unlocked the potential to cheaply and easily specialize large pretrained models. 
  However, the most prominent approaches, like low-rank adapters (LoRA) depend on heuristics or rules-of-thumb for their architectural choices---potentially limiting their performance for new models and architectures. 
  This limitation suggests that techniques from neural architecture search could be used to obtain optimal adapter architectures, but these are often expensive and difficult to implement. 
  We address this challenge with \underline{Mo}narch \underline{Re}ctangular Fine-tuning (MoRe), a simple framework to search over adapter architectures that relies on the Monarch matrix class.  
  Theoretically, we show that MoRe is more expressive than LoRA. 
  Empirically, our approach is more parameter-efficient and performant than state-of-the-art PEFTs on a range of tasks and models, with as few as 5\% of LoRA's parameters. 
\end{abstract}

\section{Introduction}
Large pretrained `foundation' models \cite{Bommasani2021OnTO} were originally conceived as a convenient base for rapidly building applications. 
The size and complexity of these models; however, paradoxically often made specialization more complex and challenging than traditional machine learning. 
Recently, adapters, like the popular LoRA~\cite{hu2021lora}, have dramatically decreased the cost of specialization.
This has unlocked the potential of foundation models for efficient use in everyday settings

Despite their popularity, parameter-efficient adapter techniques make particular architectural assumptions, such as the eponymous low rank in LoRA~\cite{hu2021lora, hu2023llm, zhang2023adaptive, chavan2023oneforall, liu2024parameterefficient}.
These assumptions are a good fit for certain models, tasks, and datasets---but may result in poor performance on others.
There has been a resulting arms race of parameter-efficient fine-tuning (PEFTs) techniques, each with their own benefits and drawbacks.
This suggests that the right adapter architecture should be \emph{learnable}.

Learning architectures is the traditional domain of neural architecture search (NAS). 
Unfortunately, most NAS techniques are heavyweight~\cite{pmlr-v80-pham18a, liu2018darts, pmlr-v115-li20c, li2021geometryaware}, creating a tension: 
NAS may learn better adapter architectures for a particular task but costs substantially more compute---sacrificing much of the benefits of adapters in the first place. 

We show how to resolve this tension by relying on the Monarch matrix class~\cite{dao2022monarch}. 
This class presents a simple parametrization that can express a vast range of \emph{structured matrices}, enabling conveniently learning a wide variety of parameter-efficient architectures. 
In other words, building adapters from Monarch matrices simultaneously produces two benefits---\textbf{flexibly searching over architectures} and \textbf{efficient training for adapters}. 

Based on this idea, we introduce a simple PEFT framework called \underline{Mo}narch \underline{Re}ctangular Fine-tuning (MoRe).
We study its expressiveness properties theoretically and validate it empirically.
When fixing block configuration after extensive architectural ablations, the most performant adapter we produced via MoRe is $10\times-20 \times$ more parameter-efficient than LoRA and has the \textbf{fewest tunable hyperparameters} among all PEFTs. We have open-sourced all code to reproduce our experiments on Github.

\section{Related Work}
PEFT methods trade off mechanisms for parameter-efficiency and performance. 
These mechanisms are designed heuristically, and may not be the best choice for all settings. 
Popular methods such as LoRA \cite{hu2021lora} may not strike the best tradeoff between efficiency and performance. 
Other methods often sacrifice increased complexity and a reliance on search for improved performance, limiting scalability. 
Methods such as GLoRA and AdaLora~\cite{chavan2023oneforall, zhang2023adaptive} require expensive search for their rank and block configurations. 
Our goal is to strike a better tradeoff compared to current PEFT techniques, all while avoiding expensive search procedures. 
We describe the relation between MoRe and existing techniques below.

%

\textbf{Orthogonal Fine-tuning.} 
Butterfly Orthogonal Fine-tuning (BOFT) \cite{liu2024parameterefficient} uses a compute-heavy neighbor of Monarch, the butterfly matrices, to parameterize Cayley-transformed orthogonal blocks that are multiplied with the original weight matrices. 
It has two more tunable parameters, the block number and size.   
In contrast, MoRe does not require tuning the number of blocks or rank and is more performant and parameter-efficient than BOFT.

\textbf{Representation Editing.}  
ReFT \cite{wu2024reft} is a prompt editing method operating on low-rank subspaces. 
It balances expressivity and parameter efficiency by only intervening on selected layers and tokens, often surpassing PEFTs that adapt model weights. 
However, it induces inference overhead and an even larger search space than LoRA (token positions and layers to intervene). 
It is also somewhat less well-understood compared to existing PEFTs from a theoretical point of view. 
%

\section{MoRe Framework}

Monarch matrices~\cite{dao2022monarch} are a rich collection of block-sparse structured matrices that subsume butterfly matrices and belong to the Kaleidoscope matrices~\cite{dao2020kaleidoscope}, a class of matrices that can represent any structured matrix and a variety of transforms such as the Fourier transform, cosine transforms, and Hadamard transform. Unlike structure-agnostic matrix families, such as those of low rank, Monarch matrices can have arbitrary rank, and their products are not closed, allowing \emph{for a richer matrix class} as more Monarch matrices are multiplied.


\begin{figure}
    \centering
    \includegraphics[width=8cm]{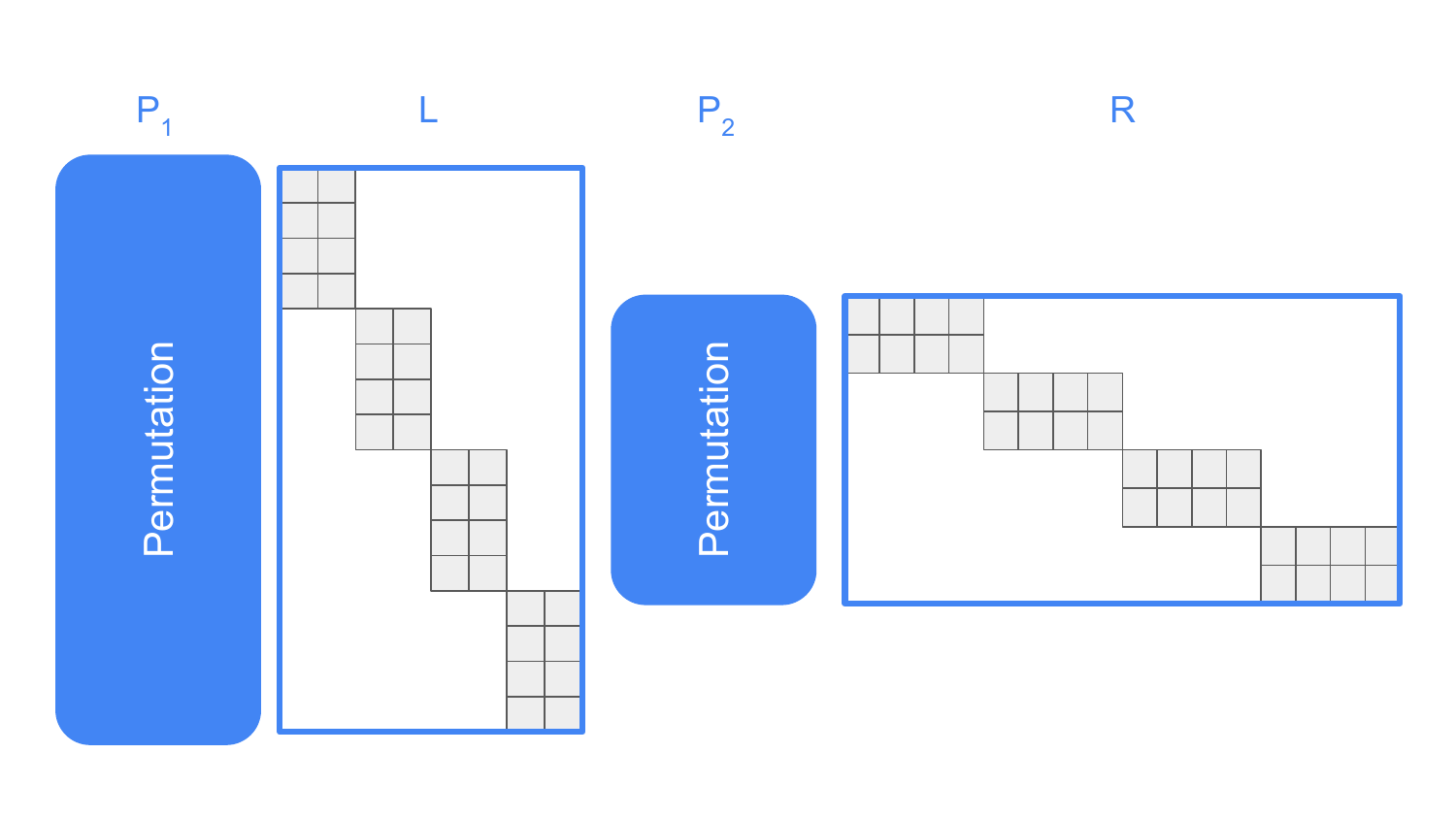}
    \caption{
    The structure of low-rank Monarch matrices contains two permutations $P_1$ and $P_2$ along with two block-diagonal components $L$ and $R$ which are learned while $P_1$ and $P_2$ are both fixed. 
    In the above, the number of blocks $N=4$, with input dimension $n=16$, and the block `rank' is $r_{blk}=2$ and size is $n/N=4$. \textbf{The pseudo-code can be found in appendix \ref{algo:monarch}.}
    }
    \label{fig:monarch}
\end{figure}

Let $n$ be the dimensions of the Monarch matrix $M$, i.e. $M \in \mathbb{R}^{n\times n}$. Define $N$ as the number of blocks in component matrices $L$ and $R$ and $r_{blk}$ as the rank of each sub-block. In Monarch matrix's standard square form, $r_{blk} = n/N$. Monarch matrices have the following structure:
\begin{equation}
\label{eq:monarch}
    M = P_1 L P_2 R, 
\end{equation}
where $P_1$ and $P_2$ are permutation matrices and $L$ and $R$ are block diagonal (see Figure \ref{fig:monarch}).

Original work in Monarch matrices focused on the case where $L$ and $R$ are block-wise square, but the family of Monarch matrices is more general and includes \textit{low-rank} Monarch matrices. This extension allows for $L$ and $R$ to be rectangular, with similarly shaped block diagonal components. This allows the overall rank of the Monarch matrix to be constrained by forcing $L$ and $R$ to have similar shapes to LoRA components---but with fewer parameters, as Monarch only contains non-zero entries within the diagonal blocks.

\begin{figure}
    \centering
    \includegraphics[scale=0.7]{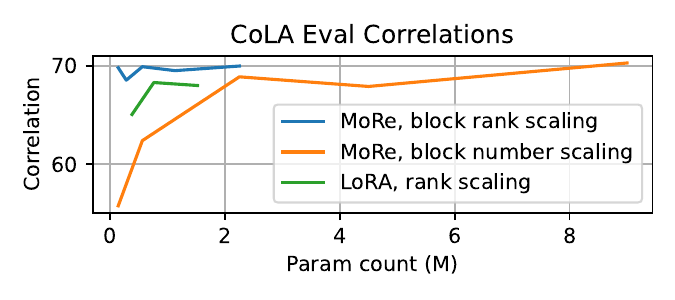}
    \caption{
    Matthew's Correlation on CoLA when trade parameter counts for performance on two axes: the block dimension and the number of blocks, both with square blocks. 
    The block dimensions used are $[4, 8, 16, 32, 64]$ and the $N$ are $[1024, 256, 128, 32, 16]$. 
    }
    \label{fig:trade_off}
\end{figure}
\begin{figure}
    \centering
    \includegraphics[scale=0.7]{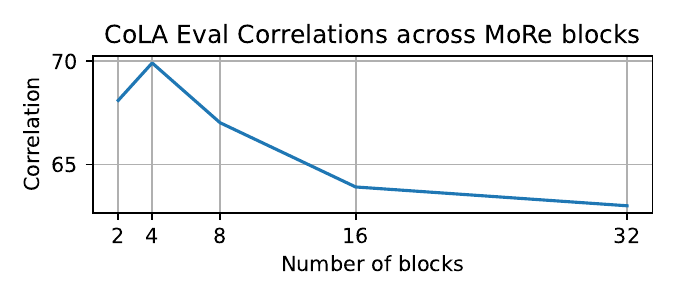}
    \caption{
    Fixing $r_{blk} = 4$, increasing the number of blocks beyond $4$ does not lead to better performance. 
    }
    \label{fig:tune_nblocks}
\end{figure}

\noindent \textbf{MoRe Fine-Tuning.} During training, for a pretrained weight matrix \( W \in \mathbb{R}^{m \times n} \) and bias \(b \in \mathbb{R}^{m}\), we apply MoRe via:
\begin{equation}
\label{eq:training}
    \Phi_{MoRe}(x) = Wx + Mx + b,
\end{equation}
and update $L$ and $R$, where $L$ has shape ($N$, $r_{blk}$, $n/N$) and $R$ has shape ($N$, $n/N$, $r_{blk}$).
During inference, $W$ absorbs $M$ as in LoRA so there is zero additional overhead.

\begin{table*}[t!]
\centering
\resizebox{\linewidth}{!} {
\begin{tabular}{lcccccccccc}
    \toprule
    \textbf{Method} & \textbf{\#Params} & \textbf{BoolQ} & \textbf{PIQA} & \textbf{SIQA} & \textbf{HellaS.} & \textbf{WinoG.} & \textbf{ARC-e} & \textbf{ARC-c} & \textbf{OBQA} & \textbf{Avg.} \\
    \midrule
    LoRA\(_{r=32}\) & 53.3M (0.830\%) & 68.9 & 80.7 & 77.4 & 78.1 & 78.8 & 77.8 & 61.3 & 74.8 & 74.7 \\
    LoRA\(_{r=32}\), \textbf{Llama 13B}  & 83.2M (0.670\%) & 72.1 & 83.5 & 80.5 & 90.5 & 83.7 & 82.8 & 68.3 & 82.4 & 80.5 \\
    \textbf{MoRe\(_{r=32}\); q, k, v (ours)}  & \textbf{3M (0.047\%)} & 67 & 86.4 & \textbf{88.4} & \textbf{97.3} & \textbf{95.1} & \textbf{88.5}  & \textbf{76.6} & 79.9 & \textbf{84.9} \\
    ReFT &  \textbf{2.0M (0.031\%)} & 69.3 & 84.4 & 80.3 & 93.1 & 84.2 & 83.2 & 68.2 & 78.9 & 80.2 \\
    ReFT, \textbf{Llama 13B} &  3.1M (0.025\%) & 72.1 & \textbf{86.3} & 81.8 & 95.1 & 87.2 & 86.2 & 73.7 & \textbf{84.2} & 83.3 \\
    $\text{Adapter}^\text{S*}$ & 99.3M (0.800\%) & 71.8 & 83.0 & 79.2 & 88.1 & 82.4 & 82.5 & 67.3 & 81.8 & 79.5 \\
    $\text{Adapter}^\text{P*}$ & 358.7M (2.890\%) & 72.5 & 84.9 & 79.8 & 92.1 & 84.7 & 84.2 & 71.2 & 82.4 & 81.5 \\
    DoRA (half)* & 43.4M (0.350\%) & \textbf{72.5} & 85.3 & 79.9 & 90.1 & 82.9 & 82.7 & 69.7 & \textbf{83.6} & 80.8 \\
    DoRA  & 84.4M (0.680\%) & 72.4 & 84.9 & 81.5 & 92.4 & 84.2 & 84.2 & 69.6 & 82.8 & 81.5 \\
    \hline
    ChatGPT  & -- & 73.1 & 85.4 & 68.5 & 78.5 & 66.1 & 89.8 & 79.9 & 74.8 & 77.0 \\
    \bottomrule
\end{tabular}
}
\caption{Commonsense reasoning results. We take all numbers except for MoRe from \citet{liu2024dora}. Llama 1 7B is used unless otherwise specified.}
\label{tab:commonsense}
\end{table*}

\begin{table*}[t!]
\centering
\small
\resizebox{\linewidth}{!}{
\begin{tabular}{lcccccccccc}
    \toprule
    \textbf{PEFT} & \textbf{\#Params} & \textbf{AQuA} & \textbf{GSM8K} & \textbf{MAWPS} & \textbf{SVAMP} & \textbf{Avg.} \\
    \midrule
    LoRA\(_{r=32}\) & 53.3M (0.830\%) & 18.9 & \textbf{37.5} & 79.0 & \textbf{52.1} & 46.9 \\
    \textbf{MoRe\(_{r=32}\); q, k, v (ours)}  & \textbf{3M (0.047\%)} & 22.1 & 28.5 & 84.3 & 48.4 & 45.8 \\
    \textbf{MoRe\(_{r=32}\) (ours)}  & 10.68M (0.166\%) & \textbf{24.0} & 29.6 & \textbf{85.7} & 48.7 & \textbf{47.0} \\
    ReFT &  \textbf{1.99M (0.031\%)} & 21.4 & 26.0 & 76.2 & 46.8 & 42.6 \\
    PrefT* &  7.1M (0.110\%) & 14.2 & 24.4 & 63.4 & 38.1 & 35.0 \\
    $\text{Adapter}^\text{S*}$ & 63.6M (0.990\%) & 15.0 & 33.3 & 77.7 & 52.3 & 44.6 \\
    $\text{Adapter}^\text{P*}$ & 227.5M (3.540\%) & 18.1 & 35.3 & 82.4 & 49.6 & 46.4 \\
    \bottomrule
\end{tabular}
}
\caption{Math reasoning results on Llama 1 7b. We take all baseline results from \cite{hu2023llm}. }
\label{tab:math}
\end{table*}

Monarch matrices were originally proposed to accelerate pre-training, using two block-wise square monarch factors to substitute one dense matrix multiplication, with $O(n\sqrt{n})$ FLOPs. However, an interesting property of these matrices from rectangular factors is that even though each block is constrained to rank $r_{blk}$, the overall product can have a rank as large as $r$ = $Nr_{blk}$. We set $N$ to 4 for the best rank-sparsity balance. MoRe can achieve the same rank as LoRA with far fewer parameters, which empirically translates to added fine-tuning performance.

\subsection{Architectural Choices \& Analysis}
Inspired by work on search-based adaptation: AdaLoRA~\cite{zhang2023adaptive} which adaptively allocates parameters for different ranks, and GLoRA~\cite{chavan2023oneforall}, which tunes the adapter complexity layer-wise, we explored different adapter styles (see Appendix \ref{ablation}) as well as trading sparsity ($N$) and rank ($r_{blk}$) for the best investment of our parameter budget (Figure \ref{fig:trade_off}). Since merely changing $N$ does not change the parameter count, we constrained each block to be square for \textbf{block number scaling}. 

Interestingly, our search converged to a minimal 4-block architecture with \textbf{\textit{the fewest tunable hyperparameters among all methods}}, without the adapter scaler $\alpha$ in LoRA. Our search space trivially subsumes LoRA if we set $N$ to 1. Empirically, MoRe with $N=1$ and $r=r_{blk}=8$ obtains 68.18 Matthew's Correlation on CoLA, aligning with the 68.3 for rank 8 LoRA.

\noindent \textbf{Should we tune the number of blocks?} 
Due to our rectangular block-diagonal parametrization, increasing $N$ while fixing $r_{blk}$ increases the total rank $r$ under the same parameter budget. However, this induces worse performance, possibly because the matrix is sparser and it is harder to converge to a stable subspace. Empirically, performance drops drastically when $N>4$ (Figure~\ref{fig:tune_nblocks}).

\noindent \textbf{Relationship To BOFT.} 
BOFT \cite{liu2024parameterefficient} uses butterfly matrices~\cite{Dao2019LearningFA, dao2020kaleidoscope}, a related class of structured matrices that represent recursive division in FFT with learnable coefficients to efficiently perform feature(channel) mixing. 
Monarch~\cite{dao2022monarch} was proposed to replace butterfly matrices due to their hardware-unfriendly sparsity patterns, and is more generalized and expressive. \cite{dao2022monarch} proves that the product of Monarch matrices strictly subsume that of the butterfly matrix class.
While butterfly matrices have $O(n\log{n})$ FLOPs, BOFT is empirically 2x slower than LoRA and occupies much more memory, which we show in the following. 

\noindent \textbf{Theoretical Analysis.} One advantage of MoRe is that it is amenable to a theoretical analysis of its expressivity, mirroring that of LoRA \cite{zeng2024expressive}. We show in Appendix \ref{section:theory} that MoRe is more expressive than LoRA. 
\section{Experimental Results}

\begin{table*}[t!]
\centering
\begin{tabular}{lcccccccccc}
    \toprule
    \textbf{PEFT} & \textbf{\#Params} & \textbf{MNLI} & \textbf{SST-2} & \textbf{MRPC} & \textbf{CoLA} & \textbf{QNLI} & \textbf{QQP} & \textbf{RTE} & \textbf{STS-B} & \textbf{Avg.} \\
    \midrule
    LoRA\(_{r=8}\)  & 0.79M & 90.2 & 96.0 & 89.8 & 65.5 & 94.7 & 90.7 & 86.3 & 91.7 & 88.16 \\
    \textbf{MoRe\(_{r=32}\) (ours)}  & 0.56M & \textbf{90.77} & \textbf{96.36} & \textbf{90.93} & \textbf{68.69} & 94.78  & 91.00  & 85.92 & 92.08 & \textbf{88.8} \\
    \textbf{MoRe\(_{r=4}\) (ours)}   & \textbf{0.14M} & 89.69  & 96.18 & 89.71 & 67.54 & \textbf{94.85} & 90.41 & 85.08 & 91.77 & 88.15 \\
    ReFT  & \textbf{0.048M} & 89.2 & 96.2 & 90.1 & 68.0 & 94.1 & 88.5 & \textbf{87.5} & 91.6 & 88.2 \\
    BOFT\textsuperscript{\smash{\(\substack{m=4 \\ b=4}\)}}  & 1.266M & 89.45 & 95.8 & 90.21 & 64.79 & 94.31 & 88.37 & 85.92 & \textbf{92.26} & 87.64  \\
    $\text{Adapter}^\text{*}$ & 0.89M & 90.1 & 95.2 & 90.5 & 65.4 & 94.6 & \textbf{91.4} & 85.3 & 91.5 & 88.0 \\
    $\text{Adapter}^\text{FFN}$ & 0.79M & 90.3 & 96.1 & 90.5 & 64.4 & 94.3 & 91.3 & 84.8 & 90.2 & 87.7 \\
    RED          & 0.048M & 89.5 & 96.0 & 90.3 & 68.1 & 93.5 & 88.8 & 86.2 & 91.3 & 88.0 \\
    \bottomrule
\end{tabular}
\caption{
Language understanding comparisons. 
We take LoRA, Adapter, ReFT and RED results from \citet{wu2024reft}. All runs are averaged over 3 random seeds. 
We report Pearson Correlation for STS-B, Matthew's correlation for CoLA, matched accuracy for MNLI, and accuracy for other tasks. 
We report different parameter counts for BOFT from the original paper---while it is claimed that due to skew-symmetricity, only half of the matrix parameters are needed, then negated and copied to the other half, in practice the whole matrix requires gradients.
}
\label{tab:glue}
\end{table*}
We conducted experiments on three challenging NLP tasks covering over 20 datasets: commonsense reasoning, math reasoning, and language understanding of models ranging from Roberta-large to Llama 7B. 
We follow the widely adopted dataset settings in LLM-Adapters \cite{hu2023llm} and ReFT \cite{wu2024reft}. All experiments are performed on a single NVIDIA A100 40G, and use Flash Attention \cite{dao2022flashattention} when applicable. As we shall see, besides fixing $N$, MoRe needs \textbf{almost no tuning for rank $r_{blk}$}.

\noindent \textbf{Commonsense Reasoning.}
We train the Llama 1 7b model \cite{touvron2023llama} on the challenging Commonsense170k benchmark in \cite{hu2023llm} consisting of eight commonsense reasoning tasks. The model is prompted with multiple-choice problems to output the correct choice without step-wise reasoning. We report accuracy on the test set in table \ref{tab:commonsense}, and hyperparameter details can be found in \ref{hp}. \textbf{Note that MoRe with Llama 7B largely surpasses the state-of-the-art ReFT with Llama 13B with around $1/6$ of its training steps (3 epochs).}

\noindent \textbf{Math Reasoning.}
We train Llama 1 on the Math 10k dataset consisting of seven complex math reasoning tasks from \citet{hu2023llm}. Following \citet{wu2024reft}, we only used 4 datasets for final evaluation to avoid data leakage. The results are shown in Table~\ref{tab:math}.

\noindent \textbf{Natural Language Understanding.}
We evaluate MoRe on the GLUE benchmark \cite{glue} to show its superior parameter efficiency on small LLMs. We fine-tune RoBERTa-large 350M \cite{roberta} on eight datasets consisting of tasks such as sentiment classification and natural language inference, and report performance on the evaluation set following \citet{hu2021lora} over 3 different random seeds. Classifier heads are excluded from the parameter count. We use fp32 for all GLUE tasks, and hyperparameter tuning is done for each task separately (Appendix \ref{hp}). By default, we adapt query, key, and values. Notably, MoRe is on par with LoRA even with $r_{blk}=1$ and 0.14M parameters, and outperforms all other methods when $r_{blk}=4$.

\noindent \textbf{Memory Cost and Runtime.} 
Modern GPUs rely on the tensor cores to accelerate matrix multiplication. MoRe leverages optimized CUDA batched matrix multiplication (BMM) kernels to populate the tensor core with many small matrices, with on par or better performance than GEMM. Here we show how our training speed compares with BOFT\footnote{BOFT's public implementation does not support bf16 and fp16, so we added these features.} and LoRA in Table \ref{tab:benchmark}, using the setting in our training experiments (see Appendix \ref{hp}). For Llama, we apply bf16, flash attention, and adapt all linear modules by default. Notably, \textbf{\textit{BOFT runs out of memory even on H100 80G}}, rendering it impractical for large models. 
%
MoRe slightly lags behind LoRA for the 350M RoBERTa due to the overhead of permutations allocating extra memory and multiple CUDA kernel launches, which we will address in a future Triton implementation, but excels in larger models due to its parameter efficiency.
\begin{table}[]
\centering
\resizebox{\linewidth}{!}{%
\begin{tabular}{lcccc}
    \toprule
    \textbf{Model} & \textbf{PEFT} & \textbf{Task} & \textbf{Peak Memory} & \textbf{Runtime} \\
    \midrule 
    RoBERTa-large & BOFT\textsuperscript{\smash{\(\substack{m=4 \\ b=4}\)}} & CoLA & 5.98 GB & 29.9 min \\
    RoBERTa-large & LoRA\(_{r=8}\) & CoLA & 4.3 GB & 14.7 min \\
    RoBERTa-large & MoRe\(_{r=32}\) & CoLA & 5.68 GB & 15.5 min \\
    Llama 7b & BOFT\textsuperscript{\smash{\(\substack{m=4 \\ b=4}\)}}; q, k, v & Math & 53.97 GB & 10 hr \\
    Llama 7b & BOFT\textsuperscript{\smash{\(\substack{m=4 \\ b=4}\)}} & Math & OOM & OOM \\
    Llama 7b & LoRA\(_{r=32}\) & Math & 20.9 GB & 4.83 hr \\
    Llama 7b & MoRe\(_{r=32}\) & Math & \textbf{20.6 GB} & \textbf{4.55 hr} \\
    \bottomrule
\end{tabular}}
\caption{Comparison of peak memory and runtime. We use a batch size of 2 for Llama 7B and 16 for RoBERTa. Due to the prohibitive memory cost of BOFT, we use H100 to benchmark Llama and only adapted query, key, and value for BOFT.}
\label{tab:benchmark}
\vspace{-2em}
\end{table}

\section{Conclusion}
We introduced MoRe, a framework for searching for high-quality adapter architectures via Monarch matrices.
MoRe offers excellent performance and has multiple promising directions for future work (described in Appendix \ref{section:limitation}). 
%


\bibliography{refs}
\bibliographystyle{icml2024}
\clearpage
\newpage

\appendix
\section*{Appendix}

\section{Theoretical results}\label{section:theory}
We show a theoretical finding for the expressiveness of MoRe in the spirit of \citet{zeng2024expressive}.
    
We start with a simple result.
\begin{lemma}\label{lma:submatrix_norm}
        Let $W$ be an $n\times n$ matrix, where $n=m^2$ for some integer $m$. Let $W_{jk}$ denote the submatrix of $W$ such that 
        \begin{equation*}
            W = \begin{bmatrix}
                W_{11} & W_{12} & \hdots & W_{1m} \\
                W_{21} & W_{22} & \hdots & W_{2m} \\
                \vdots & \vdots &         & \vdots \\
                W_{m1} & W_{m2} & \hdots & W_{mm} \\
            \end{bmatrix}
        \end{equation*}
        Let $x\in \mathbb{R}^n$, with a similar decomposition into $x_k$ for $k=1,2,\hdots,m$. Then $\|Wx\|_2 \leq \sum_{jk}\|W_{jk}x_k\|_2$.
    \end{lemma}

    \begin{proof}
    We have that 
        \begin{align*}
            \|Wx\|_2 &= \|\begin{bmatrix}
                W_{11} & \hdots & W_{1m} \\
                \vdots & \ddots & \vdots \\
                W_{m1} & \hdots & W_{mm} \\
            \end{bmatrix} \begin{bmatrix}
                x_1 \\
                \vdots \\
                x_n
            \end{bmatrix}\|_2 \\
            &= \|\begin{bmatrix}
                W_{11}x_{1} + \hdots +  W_{1m}x_{m} \\
                \vdots \\
                W_{m1}x_{1} + \hdots +  W_{mm}x_{m} \\
            \end{bmatrix}\|_2 \\
            &\leq \sum_j \|W_{j1}x_{1} + \hdots +  W_{jm}x_{m}\|_2 \\
            &\leq \sum_{jk} \|W_{jk}x_{k}\|_2 \\
        \end{align*}
    \end{proof}

    \begin{corollary}
        Let $W$ be an $n\times n$ matrix, where $n=m^2$ for some integer $m$. Let $W_{jk}$ denote the submatrices of $W$. Then $\sigma_1(W) \leq \sum_{jk} \sigma_1(W_{jk})$.
    \end{corollary}

    \begin{proof}
        Note that $\sigma_1(M) = \|M\|_2$ for any matrix $M$. Let $x\in \mathbb{R}^n$ such that $\|x\|_2=1$ and $\|Wx\|_2=\|W\|_2$. Using Lemma \ref{lma:submatrix_norm}, 
        \begin{equation*}
            \|W\|_2 = \|Wx\|_2 \leq \sum_{jk} \|W_{jk}x_{k}\|_2
        \end{equation*}
        Since $\|W_{jk}x_{k}\|_2 \leq \|W_{jk}\|_2 \cdot  \|x_{k}\|_2 \leq \|W_{jk}\|_2$, we have 
        \begin{equation*}
            \|W\|_2 \leq \sum_{jk} \|W_{jk}x_{k}\|_2 \leq \sum_{jk} \|W_{jk}\|_2
        \end{equation*}
    \end{proof}

    \begin{theorem}\label{thm:est_err_last}
        Suppose both the target and frozen model are linear and respectively parameterized by $\bar{W}$ and $W=\prod_{l=1}^L W_l$ with both $\bar{W}$ and $W$ full rank. Assume that $r=N$ for these Monarch matrices, i.e. the Monarch factors are square with square blocks on the diagonal. The adapted model is allowed to fine-tune the last layer's parameters with a Monarch matrix: $\hat{W}=\prod_{l=1}^{L-1}W_l (W_L+\Delta_{W_L})$, where $\Delta_{W_L} \in \mathcal{M}$. Define error between the target and the frozen model as $E = \bar{W}-W$, and regularized error as $\Tilde{E} = (\prod_{l=1}^{L-1}W_l)^{-1}E$.
        The estimation error between the adapted model and the target model is bounded:
        \begin{align*}
            \|\bar{W}-\hat{W}\|^2_F &\leq \|\prod_{l=1}^{L-1}W_l\|_F^2 \cdot \|\Tilde{E}-\Delta_{W_L}\|_F^2 \\
            &= \|\prod_{l=1}^{L-1}W_l\|_F^2 \cdot (\sum_{jk}(\sum_{i=2}{\sigma_i^2(\Tilde{E}_{:,j,k,:})})).
        \end{align*}
        where $\sigma_i$ is the $i$-th eigenvalue of the given function and $\Tilde{E}_{ijkl}$ is $\Tilde{E}$ reshaped into a 4-D tensor.
    \end{theorem}

    \begin{proof}
        The proof directly follows the decomposition in the Monarch paper \cite{dao2022monarch} and the previously derived results. 
    \end{proof}
        
    We now use a worst-case to illustrate how the Monarch approximation differs from a rank-1 approximation. Let $A$ be any matrix of size $n\times n$. Reshape $A$ into a 4D tensor $\Tilde{A}$ of dimension $m\times m \times m \times m$, where $m=\sqrt{n}$. Then in the worst case, each sub-matrix is of full-rank $m$ and the singular values are all equal. An optimal monarch matrix $M$ in Frobeneus norm performs a rank-1 approximation for each sub-matrix. The estimation error $\|A-M\|^2_F$ can be interpreted as all unexplained singular values, whose proportion is $\frac{m-1}{m}$. Hence $\|A-M\|^2_F = \frac{m-1}{m}\|A\|_F^2$. This provides a bound when in a general case.

    Now consider a rank-1 approximation of $A$. In the worst case, since $A$'s rank cannot be smaller than $m$ (a full matrix's rank is always equal or greater than the rank of its sub-matrix), let $A$ be of rank $m$.  Suppose $A$'s non-zero singular values are still all equal (?). The estimation error for a rank-1 approximation $D$ of $A$ will be $\|A-D\|_F^2=\frac{m-1}{m}\|A\|_F^2$, which equals the Monarch approximation. However, in other cases where $A$'s rank is greater than $m$, a Monarch approximation is strictly better than a rank-one approximation.

    \subsection{Optimizations for Rectangular Monarch matrices}
    There are two distinct cases with variable-rank Monarch matrices. Each case depends on how the block rank compares to the block number.
    
    Let $n$ be the dimensions of $M$, the Monarch product, let $N$ be the number of blocks in each factor $L$ and $R$, let $m=n/N$ be the block width, and let $r$ be the block rank. In total, we have that $M$, $L$, and $R$ are of dimension $(n,n)$, $(n,r)$, and $(r,n)$ respectively. When compressed into 4- and 3-tensors, these have dimension $(m,N,N,m)$, $(N,r,m)$, and $(N,m,r)$ respectively. To investigate the behavior of $M=P_1LP_2R$, consider a vector $x\in \mathbb{R}^n$ and how $M$ transforms this vector. Reshape $x$ into a 2-tensor with dimensions $(N,m)$. 

    First, assume $N \geq r$ where $r$ divides $N$ and let $b=N/r$. For this case, further reshape $M$, $L$, $R$, and $x$ into shapes $(m,r,b,r,b,m)$, $(r,b,r,m)$, $(r,b,m,r)$, and $(r,b,m)$, which is possible since $rb=N$. 

    \begin{itemize}
        \item Apply $R$: This results in the intermediate $y_{kbj}=\sum_i R_{kbji} x_{kbi}$.
        \item Apply $P_2$: This transposes the first and third coordinates of $y$, so $y_{kbj} \xrightarrow{} y_{jbk}$.
        \item Apply $L$: This results in the intermediate $z_{jbl}=\sum_k L_{jblk} y_{jbk}$.
        \item Apply $P_1$: This again transposes the first and third coordinates of $z$, so $z_{jbl} \xrightarrow{} z_{lbj}$.
    \end{itemize}

    In total, this amounts to computing $z_{lbj}=\sum_{k,i} L_{jblk} R_{kbji} x_{kbi}$. We then can define $M_{ljbkbi} = L_{jblk} R_{kbji}$ which defines the operation $z_{lbj}=\sum_{k,i} M_{ljbkbi} x_{kbi}$. Notice that the optimal solution can be found through a collection of rank-1 decompositions of $M_{:jbkb:}$, each of size $(m,m)$. This common index $b$ implies that whenever those coordinates in the 6-tensor disagree, this Monarch product contains zeros, so this decomposition will be sparse.
\begin{figure}
    \centering
\includegraphics[width=0.8\linewidth]{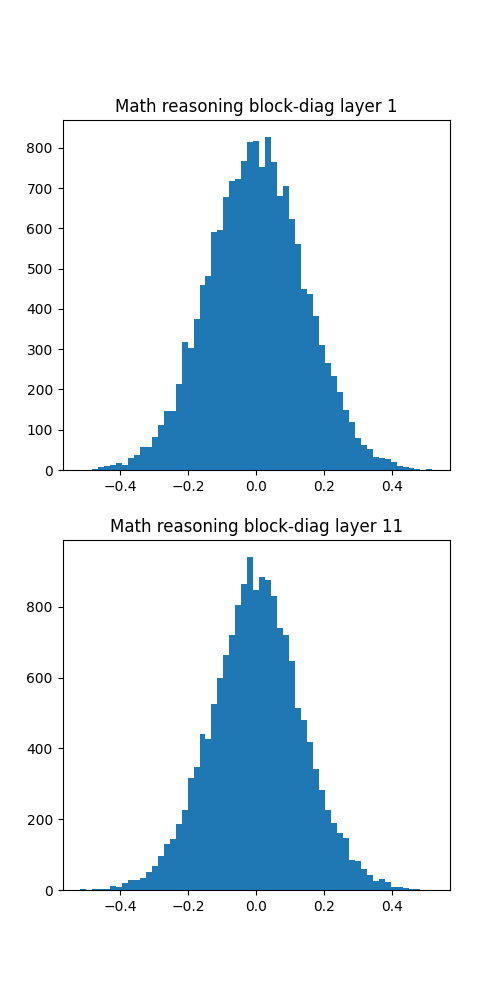}
    \caption{Llama 7b trained on Math Reasoning tasks}
    \label{fig:math_weights}
\end{figure}
    Next, assume that $N < r$ where $N$ divides $r$ and let $b=r/N$. For this case, further reshape $L$ and $R$ into shapes $(N,N,b,m)$ and $(N,m,N,b)$, which is possible since $Nb=r$. 

    \begin{itemize}
        \item Apply $R$: This results in the intermediate $y_{kjb}=\sum_i R_{kjbi} x_{ki}$.
        \item Apply $P_2$: This transposes the first and second coordinates of $y$, so $y_{kjb} \xrightarrow{} y_{jkb}$.
        \item Apply $L$: This results in the intermediate $z_{lj}=\sum_{k,b} L_{jlkb} y_{jkb}$.
        \item Apply $P_1$: This again transposes the first and second coordinates of $z$, so $z_{lj} \xrightarrow{} z_{jl}$.
    \end{itemize}

    In total, this amounts to computing $z_{jl}=\sum_{k,i,b} L_{jlkb} R_{kjbi} x_{ki}$. We then can define $M_{ljki} = \sum_b L_{jlkb} R_{kjbi}$ which defines the operation $z_{lbj}=\sum_{k,i} M_{ljki} x_{ki}$. Notice that the optimal solution can be found through a collection of rank-$b$ decompositions of $M_{:jk:}$, each of size $(m,m)$. 

    Using these decompositions, we obtain some straightforward extensions of Theorem \ref{thm:est_err_last}.

    \begin{theorem}
        Suppose both the target and frozen model are linear and respectively parameterized by $\bar{W}$ and $W=\prod_{l=1}^L W_l$ with both $\bar{W}$ and $W$ full rank. Assume that $N < r$ with $r$ a multiple of $N$. The adapted model is allowed to fine-tune the last layer's parameters with a Monarch matrix: $\hat{W}=\prod_{l=1}^{L-1}W_l (W_L+\Delta_{W_L})$, where $\Delta_{W_L} \in \mathcal{M}$. Define error between the target and the frozen model as $E = \bar{W}-W$, and regularized error as $\Tilde{E} = (\prod_{l=1}^{L-1}W_l)^{-1}E$.
        The estimation error between the adapted model and the target model is bounded:
        \begin{align*}
            \|\bar{W}-\hat{W}\|^2_F &\leq \|\prod_{l=1}^{L-1}W_l\|_F^2 \cdot \|\Tilde{E}-\Delta_{W_L}\|_F^2 \\
            &= \|\prod_{l=1}^{L-1}W_l\|_F^2 \cdot (\sum_{jk}(\sum_{i=r/N+1}{\sigma_i^2(\Tilde{E}_{:,j,k,:})})).
        \end{align*}
        where $\sigma_i$ is the $i$-th eigenvalue of the given function and $\Tilde{E}_{ijkl}$ is $\Tilde{E}$ reshaped into a 4-D tensor.
    \end{theorem}

    Notice the difference in the rightmost sum. The sum over $i$ starts at $r/N+1$ instead of $2$. 

Next, we provide experimental details.
\section{Hyperparameter Tuning}\label{hp}
We release our hyperparameter search procedure and code in the hope that it can help the community simplify this time-consuming process.
We use the asynchronous successive halving algorithm (ASHA) \cite{li2020system} to efficiently search and early-stop on our 8 * A100 cluster. 

\subsection{GLUE Language Understanding}
For BOFT, we took the hyperparameters for DeBERTA-v3 base on GLUE from their paper and tuned the learning rate only. For MoRe, we started from the hyperparameters in \cite{hu2021lora}  and randomly sampled the learning rate and batch size. We present the hyperparameters in table \ref{tab:glue_hp}.

\begin{table}[h!]
    \resizebox{\linewidth}{!}{
        \begin{tabular}{lcccccccccc}
            \toprule
            \textbf{Hyperparameter} & \textbf{MNLI} & \textbf{SST-2} & \textbf{MRPC} & \textbf{CoLA} & \textbf{QNLI} & \textbf{QQP} & \textbf{RTE} & \textbf{STS-B} \\
            \midrule
            learning rate & 2e-4 & 4.4e-4 & 3.2e-4 & 2.1e-4 & 3e-4 & 6.2e-4 & 5.4e-4 & 6.4e-4 \\
            batch size & 32 & 32 & 16 & 32 & 32 & 16 & 32 & 32 \\
            weight decay & \multicolumn{8}{c}{1e-3} \\
            lr scheduler & \multicolumn{8}{c}{cosine} \\
            Epochs & 10 & 10 & 20 & 20 & 20 & 10 & 20 & 20  \\
            \bottomrule
        \end{tabular}
    }
    \caption{GLUE hyperparameters}
    \label{tab:glue_hp}
\end{table}

\subsection{Math reasoning and Commonsense reasoning}
For these challenging reasoning tasks, we found performance to be less sensitive to hyperparameters. We took 1000 examples and 10,000 examples from Math10K and Commonsense170K as the tuning evaluation set, respectively. We present the hyperparameters in table \ref{tab:reasoning hp}.
\begin{table}[h!]
\resizebox{\linewidth}{!}{
\begin{tabular}{lcccccccccc}
    \toprule
    \textbf{Hyperparameter} & \textbf{Math reasoning} & \textbf{Commonsense reasoning} \\
    \midrule
    learning rate & 3e-4 & 4e-4 \\
    batch size(w/ gradient accumulation) & 64 & 16 \\
    weight decay & \multicolumn{2}{c}{0} \\
    lr scheduler & \multicolumn{2}{c}{cosine} \\
    dropout & \multicolumn{2}{c}{0.1} \\
    Epochs & 12 & 3 \\
    \bottomrule
\end{tabular}}
    \caption{Reasoning hyperparameters}
    \label{tab:reasoning hp}
\end{table}

\section{Architecture Ablations}\label{ablation}
With 3 potential architectural hyperparameters ($r_{blk}$, $N$ and whether to use square blocks) in our setup, one might ask whether we should use NAS to find the most efficient architecture.

We tested using monarch as a multiplicative factor instead of an additive factor as in BOFT, adding a scaler $\alpha$ on the adapter outputs as in LoRA and adding a scaler parameter; all underperform our default 4-block configuration. We also tried including $r_{blk}$ and $N$ in our hyperparameter search to mimic NAS, but \textbf{all runs converged to the configuration with the largest parameter count}, with marginal performance gains. Therefore we didn't pursue expensive NAS algorithms.
\begin{table}[h!]
\centering
\label{benchmark}
\begin{tabular}{lccc}
    \toprule
    \textbf{Method} & \textbf{GLUE CoLA} \\
    \midrule 
    MoRe (learnable scaler) & 41.1 \\
    MoRe ($\alpha=2$) & 0 \\
    MoRe (multiplicative factor) &0  \\
    \bottomrule
\end{tabular}
\label{tab:ablation}
\end{table}

\section{Learned Weight Distributions}\label{weight distribution}
We demonstrate in figure \ref{fig:math_weights} and \ref{fig:roberta_weights} that the trained block-diagonal matrices approximate Gaussian distribution well as the amount of training increases, in an attempt to interpret the results.

\begin{figure}
    \centering
    \includegraphics[width=0.8\linewidth]{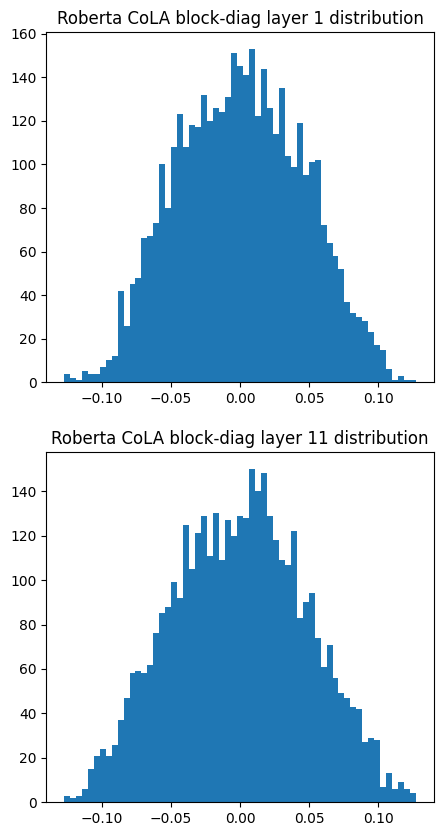}
    \caption{RoBERTa-large trained on CoLA}
    \label{fig:roberta_weights}
\end{figure}

\section{Failure Cases}\label{failure}
Inspired by \citet{meng2024pissa} that fine-tuning is strengthening some task-specific subspace components, we attempted using the dense to sparse projection algorithm (block-wise SVD) from \cite{dao2022monarch} to initialize MoRe from principal components. However, the method fails to converge on reasoning tasks and obtains only a 57.9 correlation on CoLA.

We've also tested naively replacing the low-rank projections in ReFT with a single  Monarch factor $P$ plus permutation $P_1$, which only achieved a 19.5 correlation on CoLA.


\section{Limitations and Future Work}\label{section:limitation}
Currently, MoRe poses a few limitations that we are working to address. 
\begin{enumerate}
    \item MoRe is implemented with two BMMs and two permutations, which introduces overhead due to 4 CUDA kernel launches. With machine learning compilers such as Triton\cite{tillet2019triton}, it's easy to fuse them into one kernel and recompute the activations during backward, with memory savings and speed-up. We're testing the Triton implementation's precision.
    \item We seek to substitute low-rank projections. A natural extension from our low-rank adaptation use case is to establish MoRe as a general drop-in low-rank projection module. However as shown in the Appendix~ \ref{failure}, it does not work directly with ReFT. 
    \item Projection subspace interpretation: we show (Appendix \ref{weight distribution}) that Monarch weights approach Gaussian distribution. However, we've not explored the subspace similarity between the dense and MoRe projections such as which dense components are strengthened by MoRe, due to complicated block-diagonality. Such an understanding may enable us to initialize MoRe from dense matrices' principal components as in \citet{meng2024pissa} with improved convergence and performance, and explain why scaling rank doesn't always deliver performance.
    
\end{enumerate}

\section{Pseudocode}
As the permutations $P_{1}$ and $P_{2}$ may be less intuitive, we provide a minimal PyTorch pseudocode to demonstrate their usage below.
\clearpage
\begin{lstlisting}[language=Python, label={algo:monarch}]
# Input:
#   x: (bs, n)
#   blkdiag1: (nblocks, block rank, block size)
#   blkdiag2: (nblocks, block size, block rank)
batch_shape, n = x.shape[:-1], x.shape[-1]
nblocks, blk_r, blk_sz = blkdiag1.shape
batch_dim = torch.prod(batch_shape)
x = x.reshape(bs, nblocks, blk_sz)
out1 = torch.empty(batch_dim, nblocks, blk_r).transpose(0, 1)
# (nblocks, batch_dim, blk_sz) @ (nblocks, blk_sz, blk_r) -> (nblocks, batch_dim, blk_r)
out1 = torch.bmm(x, blkdiag1.transpose(-1, -2), out=out1) 
out1 = out1.transpose(0, 1).reshape(batch_dim, blk_r, nblocks)
out1 = out1.transpose(-1, -2).contiguous().transpose(0, 1)
out2 = torch.empty(nblocks, batch_dim, blk_sz)
# (nblocks, batch_dim, blk_r) @ (nblocks, blk_r, blk_sz) -> (nblocks, batch_dim, blk_sz)
out2 = torch.bmm(out1, blkdiag2.transpose(-1, -2), out=out2) 
out2 = out2.permute(1, 2, 0).reshape(*batch_shape, blk_sz * nblocks)
\end{lstlisting}\label{algo:monarch}

\end{document}